\documentclass[11pt]{article}
\usepackage{amstext,amssymb,amsmath,epsf,graphics}

\newtheorem{lemma}{Lemma}

\newcommand{\f}{f}

\newcommand{\g}{f^{-1}}

\renewcommand{\leq}{\leqslant}

\newcommand{\bm}[1]{\boldsymbol #1}

\newcommand{\Real}{\mbox{$\mathbb{R}$}}

\begin{document}

\title{\bf Learning from the Kernel and the Range Space}

\author{\small{Kar-Ann Toh}\\
\footnotesize{School of Electrical and Electronic Engineering} \\
\footnotesize{Yonsei University}, Seoul, Korea \\
katoh@yonsei.ac.kr}

\date{IEEE/ACIS 17th ICIS 2018, pp.417--422\\
{\scriptsize Camera-ready finalized on 22 April 2018, paper presented on 07 June 2018 in ICIS} \\
{\scriptsize(Some typos have been fixed in current version)}}

\maketitle

\begin{abstract}
In this article, a novel approach to learning a complex function which can be written as
the system of linear equations is introduced. This learning is grounded upon the
observation that solving the system of linear equations by a manipulation in the kernel
and the range space boils down to an estimation based on the least squares error
approximation. The learning approach is applied to learn a deep feedforward network with
full weight connections. The numerical experiments on network learning of synthetic and
benchmark data not only show feasibility of the proposed learning approach but also
provide insights into the mechanism of data representation.
\end{abstract}

\section{Introduction}

The learning problem in machine intelligence has been traditionally formulated as an
optimization task where an error metric is minimized. In the system of linear equations,
because it is difficult to have an exact match between the sample size and the number of
model parameters, an approximation is often sought-after according to the primal solution
space or the dual solution space in the least error sense. Such an optimization,
particularly one that is based on minimizing the least squares error, has been a popular
choice due to its simplicity and tractability in analysis and implementation. The
approach is predominant in engineering applications as evident from its pervasive
adoption in statistical and network learning.

Attributed to the computational effectiveness of the backpropagation algorithm running on
the then limited hardware (see e.g.,\cite{KelleyH1,Werbos10,Werbos33,Werbos11,Haykin1})
and the theoretical establishment of the mapping capability (see e.g.,
\cite{Funa1,Hornik1,Cybenko1,HechNiel3}), the multilayer neural networks were once a
popular tool for research and applications in the 1980s. With the advancement of
computing facilities in the 1990s-2000s, such minimization of the error cost function had
been progressed to the more memory intensive search algorithms utilizing the first- and
the second-order methods of gradient descent (see e.g.,
\cite{Battiti1,Patrick11,Barnard1}). Recently, driven by another leap bound advancement
in the computing resources together with the availability of a large quantity of data,
the multilayer neural networks reemerged as deep learning networks \cite{Ian1}. In view
of the more demanding task of processing a large quantity of data with the highly complex
network function on the limited computing resources such as the operating memory and the
level of data vectorization, the backpropagation remained a viable tool for the
optimization search.

In this work, we explore into utilization of the Kernel And the Range space (abbreviated
as KAR space) for network learning. This approach exploits the approximation property of
the kernel and the range space of the system of linear equations for learning the network
weights. The main advantage of this approach is that neither descent nor gradient
computation is needed for network learning. Moreover, once the network has been
initialized, the learning can be calculated in a single operating pass where no iterative
search is needed. The proposed approach can be applied to networks of arbitrary number of
layers. This proposal opens up a new way of solving the network and functional learning
problems without having to compute the gradient.

\section{The Kernel and the Range Space}

Consider the system of linear equations given by
\begin{equation}\label{eqn_LinearEqn}
    {\bf X}\bm{w} = {\bf y},
\end{equation}
where ${\bf X}\in\Real^{m\times d}$ is the data matrix, ${\bf y}\in\Real^{m\times 1}$ is
the target vector, and $\bm{w}\in\Real^{d\times 1}$ is the unknown parameter vector to be
solved. The \emph{range} or \emph{image} of a matrix is the span of its column vectors.
The range of the corresponding matrix transformation is called the \emph{column space} of
the matrix. The \emph{kernel} or the \emph{null space} of a linear map is the set of
solutions to the homogeneous equation ${\bf X}\bm{w} = {\bf 0}$. In other words, $\bm{w}$
is in the kernel of ${\bf X}$ if and only if $\bm{w}$ is orthogonal to each of the row
vectors of ${\bf X}$.
%The kernel of a matrix is also known as the \emph{null space} or the \emph{row} space of
%a matrix.

For an under-determined system \eqref{eqn_LinearEqn} where $m<d$, the number of equations
is less than the unknowns. This gives rise to an infinite number of solutions. However, a
least norm solution can be obtained by constraining $\bm{w}\in\Real^d$ to its $\Real^m$
subspace \cite{Madych1} by utilizing the row space of ${\bf X}$, i.e., $
    \hat{\bm{w}} = {\bf X}^T({\bf X}{\bf X}^T)^{-1}{\bf y}$. Here, ${\bf X}{\bf X}^T$
constitutes the kernel and is also known as the Gram matrix.

For an over-determined system where $m>d$, the $m$ equations in \eqref{eqn_LinearEqn} are
generally unsolvable when a strict equality is desired (see e.g., \cite{StrangG1}).
However, by multiplying ${\bf X}^T$ to both sides of \eqref{eqn_LinearEqn}, the resulted
$d$ equations
\begin{equation}\label{eqn_NormalEqn}
    {\bf X}^T{\bf X}\bm{w} = {\bf X}^T{\bf y},
\end{equation}
is called the normal equation which can be rearranged to give the least squares error
solution $\hat{\bm{w}} =({\bf X}^T{\bf X})^{-1}{\bf X}^T{\bf y}$ \cite{StrangG1}.

The above observations are summarized in the following results (see e.g.,
\cite{Albert1,Adi1,SLCampbell1,Boyd1}).

\begin{lemma} \label{lemma_LS}
Solving for $\bm{w}$ in the system of linear equations of the form \eqref{eqn_LinearEqn}
in the column space (range) of ${\bf X}$ or in the row space (kernel) of ${\bf X}$ is
equivalent to solving the least squares error approximation problem. Moreover, the
resultant solution $\hat{\bm{x}}$ is unique with a minimum-norm value in the sense that
$\|\hat{\bm{w}}\|^2_2\leq\|\bm{w}\|^2_2$ for all feasible $\bm{w}$.
\end{lemma}

The proof is omitted here due to the space constraint. This result for systems with
single output (output containing a single column, ${\bf y}$) can be generalized to system
with multiple outputs (output with multiple columns, ${\bf Y} =[{\bf y}_1, \cdots, {\bf
y}_q]$) as follows.

\begin{lemma} \label{lemma_LS_matrix}
Solving for $\bm{W}$ in the system of linear equations of the form
\begin{equation}\label{eqn_LinearEqn_matrix}
    {\bf X}\bm{W} = {\bf Y}, \ \ \ {\bf X}\in\Real^{m\times d},\ \bm{W}\in\Real^{d\times q},
    \ {\bf Y}\in\Real^{m\times q}
\end{equation}
in the column space (range) of ${\bf X}$ or in the row space (kernel) of ${\bf X}$ is
equivalent to minimizing the sum of squared errors given by
\begin{equation}\label{eqn_squared_error_distance_trace}
    \textup{SSE} = trace\left( ({\bf X}\bm{W}-{\bf Y})^T({\bf X}\bm{W}-{\bf Y}) \right).
\end{equation}
Moreover, the resultant solution $\hat{\bm{W}}$ is unique with a minimum-norm value in
the sense that $\|\hat{\bm{W}}\|^2_2\leq\|\bm{W}\|^2_2$ for all feasible $\bm{W}$.
\end{lemma}
\begin{proof}
Equation \eqref{eqn_LinearEqn_matrix} can be re-written as a set of multiple linear
systems of \eqref{eqn_LinearEqn} as
\begin{equation}\label{eqn_LinearEqn_matrix_decomposed}
    {\bf X}[\bm{w}_1,\cdots,\bm{w}_q] = [{\bf y}_1,\cdots,{\bf y}_q].
\end{equation}
Since the trace of $({\bf X}\bm{W}-{\bf Y})^T({\bf X}\bm{W}-{\bf Y})$ is equal to the sum
of the squared lengths of the error vectors ${\bf X}\bm{w}_i-{\bf y}_i$, $i=1,2,...,q$,
the unique solution $\hat{\bm{W}}=({\bf X}^T{\bf X})^{-1}{\bf X}^T{\bf Y}$ in the column
space of ${\bf X}$ or that $\hat{\bm{W}}={\bf X}^T({\bf X}{\bf X}^T)^{-1}{\bf Y}$ in the
row space of ${\bf X}$, not only minimizes this sum, but also minimizes each term in the
sum \cite{Duda1}. Moreover, since the column and the row spaces are independent, the sum
of the individually minimized norms is also minimum.
\end{proof}

The process of solving the algebraic equations under the \emph{kernel-and-range} (KAR)
spaces with implicit least squares error seeking shall be exploited to solve the network
learning problem in the following section.

\section{Network Learning}

Consider an $n$-layer network (of $h_1\cdots h_n$ structure where $h_n=q$ is the output
dimension) given by
\begin{equation}\label{eqn_Nlayer_net}\small
    {\bf Y} = \f_n\left([{\bf 1},\f_{n-1}(\cdots\left[{\bf 1},\f_2([{\bf 1},\f_1({\bf
    X}\bm{W}_1)]\bm{W}_2)\right]\cdots\bm{W}_{n-1})]\bm{W}_n \right),
\end{equation}
where ${\bf X}\in\Real^{m\times (d+1)}$, $\bm{W}_1\in\Real^{(d+1)\times h_1}$,
$\bm{W}_2\in\Real^{(h_1+1)\times h_2}$, $\cdots$,
$\bm{W}_{n-1}\in\Real^{(h_{n-2}+1)\times h_{n-1}}$,
$\bm{W}_{n}\in\Real^{(h_{n-1}+1)\times q}$, ${\bf 1}=[1,...,1]^T\in\Real^{m\times 1}$,
and ${\bf Y}\in\Real^{m\times q}$.
We shall partition the term $\bm{W}_n$ into $\left[\begin{array}{c} {\bf w}_n^T \\ \textsf{W}_n \\
\end{array}\right]$ where ${\bf w}_n\in\Real^{q\times 1}$ and $\textsf{W}_n\in\Real^{h_{n-1}\times q}$.
Assume that $\g_n$ exists, we can take the inverse of $\f_n$ to both sides of
\eqref{eqn_Nlayer_net}. By separately considering the bias term and moving it to the left
hand side, we can post-multiply both sides of the equation by $\textsf{W}_n^T$ to get
{\small\begin{eqnarray} \lefteqn{
  \left[\g_n({\bf Y}) - {\bf 1}\cdot{\bf w}_n^T\right]\textsf{W}_n^T }\nonumber \\
    && = \f_{n-1}(\cdots\left[{\bf 1},\f_2([{\bf 1},\f_1({\bf
    X}\bm{W}_1)]\bm{W}_2)\right]\cdots\bm{W}_{n-1})\cdot\textsf{W}_n\textsf{W}_n^T
  \nonumber\\
\end{eqnarray} }
Next, by taking advantage of the column-row space manipulation, we arrive at
{\small\begin{eqnarray} \lefteqn{\hspace{-5mm}
  \lim_{\lambda\rightarrow 0}\left[\g_n({\bf Y}) - {\bf 1}\cdot{\bf w}_n^T\right]\textsf{W}_n^T \left(\lambda{\bf
     I}+\textsf{W}_n\textsf{W}_n^T\right)^{-1} } \nonumber\\ && =
      \f_{n-1}(\cdots\left[{\bf 1},\f_2([{\bf 1},\f_1({\bf
    X}\bm{W}_1)]\bm{W}_2)\right]\cdots\cdot\bm{W}_{n-1}) \nonumber\\
\lefteqn{\hspace{-5mm}
  \left[\g_n({\bf Y}) - {\bf 1}\cdot{\bf w}_n^T\right]\textsf{W}_n^T \left(\textsf{W}_n\textsf{W}_n^T\right)^{-1} } \nonumber\\ && =
      \f_{n-1}(\cdots\left[{\bf 1},\f_2([{\bf 1},\f_1({\bf
    X}\bm{W}_1)]\bm{W}_2)\right]\cdots\bm{W}_{n-1}) \nonumber\\
\lefteqn{\hspace{-5mm}
  \left[\g_n({\bf Y}) - {\bf 1}\cdot{\bf w}_n^T\right]\textsf{W}_n^{\dag} } \nonumber\\ && =
      \f_{n-1}(\cdots\left[{\bf 1},\f_2([{\bf 1},\f_1({\bf
    X}\bm{W}_1)]\bm{W}_2)\right]\cdots\bm{W}_{n-1}),
  \label{eqn_omega_expr}
\end{eqnarray} }
\hspace{-1.5mm}where $\dag$ denotes the pseudo-inverse operation. This pseudo-inverse can
be in the form of left or right operation depending on the matrix rank condition.

This process of inversion and moving the bias term plus kernel-and-range space
manipulation is continued until reaching the first-layer where its hidden weights can be
written as:
\begin{eqnarray}
  && \hspace{-8mm}\bm{W}_1 = {\bf X}^{\dag}\g_1\left(
        \left[\g_2(\cdots
        \left[\g_{n-1}(
        \left[\g_n({\bf Y}) - {\bf 1}\cdot{\bf w}_n^T\right]
        \textsf{W}_n^{\dag}) \right.\right.\right. \nonumber\\ &&
        \hspace{16mm}\left.\left.\left.
        - {\bf 1}\cdot{\bf w}_{n-1}^T\right]
        \textsf{W}_{n-1}^{\dag}\cdots) - {\bf 1}\cdot{\bf w}_2^T\right]
        \textsf{W}_2^{\dag} \right),    \label{eqn_W1}
\end{eqnarray}
After having $\bm{W}_1$ derived, it can be back-substituted to obtain $\bm{W}_2$ as
\begin{eqnarray}
  && \hspace{-12mm}\bm{W}_2 = \left[{\bf 1},\f_1({\bf X}\bm{W}_1)\right]^{\dag}\g_{2}\left(
        [\g_{3}(\cdots
        [\g_{n-1}(
        [\g_n({\bf Y}) \right. \nonumber\\ &&
        \hspace{-12mm} \left.
        - {\bf 1}\cdot{\bf w}_n^T ]
        \textsf{W}_n^{\dag}) - {\bf 1}\cdot{\bf w}_{n-1}^T ]
        \textsf{W}_{n-1}^{\dag}\cdots) - {\bf 1}\cdot{\bf w}_3^T]
        \textsf{W}_3^{\dag} \right)\hspace{-1mm},   \label{eqn_W2}
\end{eqnarray}
The process is iterated until the weights of the $n$th-layer is obtained:
\begin{equation}\label{eqn_Wn}\small
   \bm{W}_n
        = \left[{\bf 1},f_{n-1}(\cdots f_2([{\bf 1},\f_1({\bf X}\bm{W}_1)]\bm{W}_2)
        \cdots\bm{W}_{n-1})\right]^{\dag} \g_n({\bf Y}).
\end{equation}

Based on the above derivation, the full weights for each layer (i.e., $\bm{W}_k$,
$k=1,\cdots,n$) can be obtained in an analytic form when the weights without considering
the bias component (i.e., $\textsf{W}_k$, $k=2,\cdots,n$) is known. Here, we propose to
use a random initialization of $\textsf{W}_k$, $k=2,\cdots,n$ for solving $\bm{W}_k$,
$k=1,\cdots,n$ in network learning. We shall call this learning network \texttt{KARnet}
for convenience.

\section{Synthetic Data} \label{sec_synthetic}

In this section, we observe the behavior of the proposed network learning on three
synthetic data sets with known properties. The first set of data represents the
regression problem whereas the second and third data sets are well-known benchmarks for
classification. For all the three problems and including the following experiments, our
choice of the activation function $f$ and its inverse for each layer are respectively the
modified \texttt{softplus} function (i.e., $f(x)=\log(0.8+e^{x})$) and its inverse given
by $\g(x)=\log(e^{x}-0.8)$.

\subsection{Single Dimensional Regression Problem}

The first set of synthetic data has been generated using $y = \sin(2x)/(2x)$ based on
$x\in\{1,2,3,4,5,6,7,8\}$ for training. To simulate noisy outputs, a 20\% of variation
from the original $y$ values has been incorporated where 10 trials of the noisy
measurements are included for training as well. A two-layer network is adopted to learn
the data. Fig.~\ref{fig_1D_2layer}(a) shows the learning results for all the eight
training data points when a two-layer network uses six hidden nodes (i.e., a 6-1
structure). This is an over-determined system since there are more data samples than the
\emph{effective} number of parameters. Fig.~\ref{fig_1D_2layer}(b) shows the results when
eight hidden nodes are used for a two-layer network (i.e., a 8-1 structure). This is an
under-determined system as there are less data samples than the number of
\emph{effective} parameters (total $8+1$ including the bias's weight). Here, we note that
for the two-layer network, the system size is determined by the dimension of $\left[{\bf
1},\f_1({\bf X}\bm{W}_1)\right]$ and its rank. \footnote{According to
\eqref{eqn_W1}-\eqref{eqn_Wn}, a Moore-Penrose inverse operation is taken over the
matrices (i.e., ${\bf X}$, $\left[{\bf 1},\f_1({\bf X}{\bf W}_1)\right]$, $\cdots$,
$\left[{\bf 1},f_{n-1}(\cdots f_2([{\bf 1},\f_1({\bf X}{\bf W}_1)]{\bf W}_2) \cdots{\bf
W}_{n-1})\right]$) to relate among the weight solutions. For the 2-layer network, the
matrix $\left[{\bf 1},\f_1({\bf X}{\bf W}_1)\right]$ which is nearest to the output is of
size ${m\times (h_1+1)}$. Here, the minimum size of the hidden nodes required for this
matrix to be invertible is $(h_1+1)=m$. This in turn gives rise to an output weight of
size $(h_1+1)\times q=m\times 1$ for $q=1$. Hence, the \emph{effective} number of
parameters (or adjustable parameters) needed for data representation is hinged upon the
number of output weights which corresponds to the sample size and the output dimension
(i.e., $m\times q$).}

%\footnote{According to \eqref{eqn_W1}-\eqref{eqn_Wn}, a Moore-Penrose inverse operation
%is taken over these matrices to relate among the weight solutions. While the sample size
%($m$) determines the height of the matrix, the width of the matrix is given by the number
%of hidden nodes of the layer of interest (i.e., $h_n+1$ including the bias node). Thanks
%to the nonlinear activation which can turn singularity of outer product to full rank, the
%minimum number of adjustable parameters to represent any set of samples is determined by
%the mere product of the number of samples with the number of outputs, i.e., $mq$, whereas
%the \emph{effective} number of parameters is hinged upon the number of output weights
%which correspond to the sample size.}

Next, a five-layer network is used to learn the same set of 1D data.
Fig.~\ref{fig_1D_5layer}(a) shows the results for the network of 1-1-1-6-1 structure.
This constitutes an over-determined case where there are less \emph{effective} parameters
than data samples). Fig.~\ref{fig_1D_5layer}(b) shows the results for the network of
1-1-1-8-1 structure which is under-determined. Here, the system structure of interest is
given by \[
\left[{\bf 1},f_{4}([\cdots f_2([{\bf 1},\f_1({\bf
X}\bm{W}_1)]\bm{W}_2)\cdots)]\bm{W}_{4}\right].\]

In summary, the network is seen to find its fit through all data points including those
noisy ones for the under-determined case in both networks. However, for the
over-determined case, the network does not fit every data points due to the insufficient
number of parameters for modelling all data points. This example clearly explains the
fitting behavior of multilayer network learning.

\begin{figure}[hhh]
  \begin{center}
  \epsfxsize=10.8cm
  \epsfysize=6.8cm
  \epsffile[112    14  1306   671]{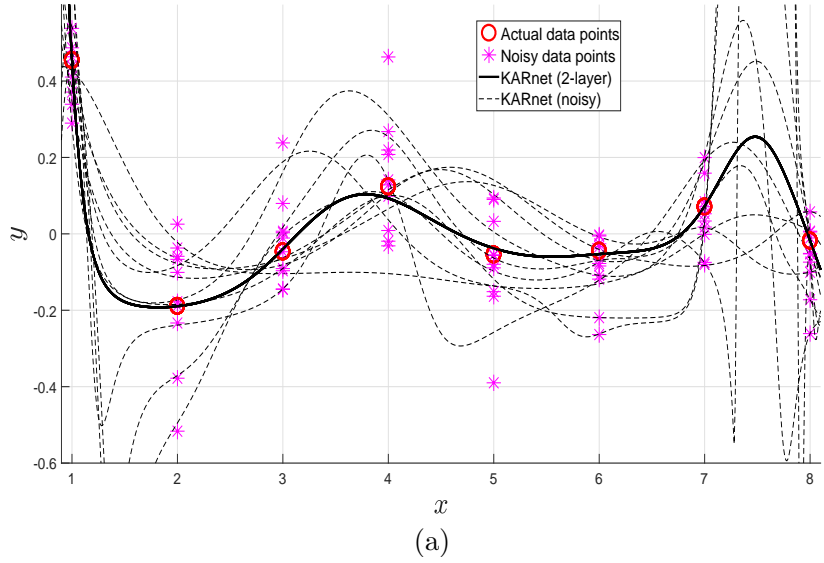}
  \\ \hspace{3mm} (a) \\*[5mm]
  \epsfxsize=10.8cm
  \epsfysize=6.8cm
  \epsffile[112    14  1306   671]{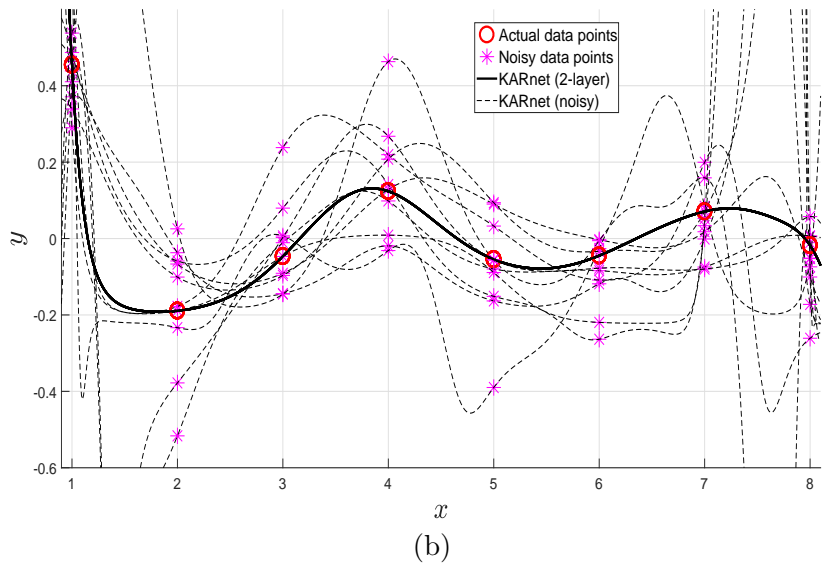}
  \\ \hspace{3mm} (b) \\
  \caption{Decision outputs of a two-layer feedforward network
  (\texttt{KARnet} with $\f=softplus$) trained by the
  proposed method. The top and bottom panels demonstrate respectively an under-fitting
  case and an over-fitting case.}
  \label{fig_1D_2layer}
  \end{center}
\end{figure}

\begin{figure}[tttt]
  \begin{center}
  \epsfxsize=10.8cm
  \epsfysize=6.8cm
  \epsffile[112    14  1306   671]{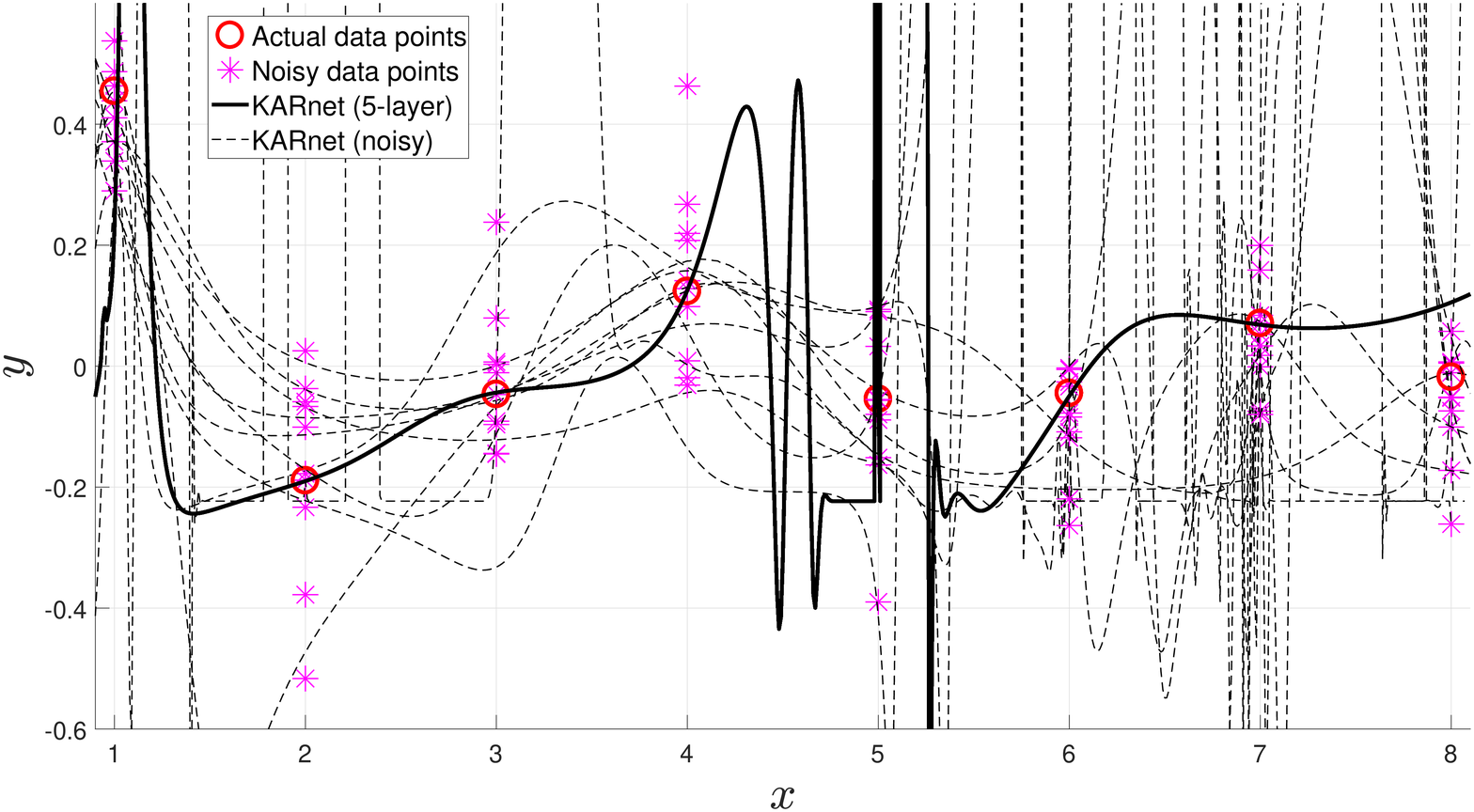}
  \\ \hspace{3mm} (a) \\*[5mm]
  \epsfxsize=10.8cm
  \epsfysize=6.8cm
  \epsffile[112    14  1306   671]{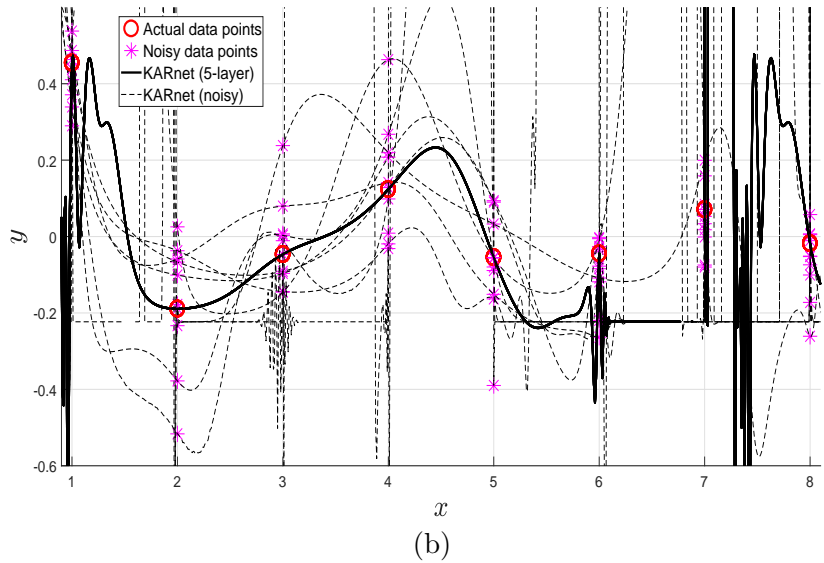}
  \\ \hspace{3mm} (b) \\
  \caption{Decision outputs of a five-layer feedforward network
  (\texttt{KARnet} with $\f=softplus$) trained by the
  proposed method. The top and bottom panels demonstrate respectively an under-fitting
  case and an over-fitting case.}
  \label{fig_1D_5layer}
  \end{center}
\end{figure}

\clearpage

\subsection{The XOR Problem}

The next example is the well-known XOR problem which consists of four data points with
one of the data points being perturbed by a small value to facilitate numerical stability
in learning (i.e., the input data points are $\{(0,0),\ (1,1),\ (1,0),\ (0.001,1.001)\}$
which are associated with labels $\{0,0,1,1\}$ respectively). A two-layer \texttt{KARnet}
with two hidden nodes is adopted to learn the data. For comparison, the
\texttt{feedforwardnet} of the Matlab toolbox is adopted with a similar architecture
(adopting a two-layer structure with \texttt{softplus} activation) for learning the same
set of data using the default training method \texttt{trainlm}.
Fig.~\ref{fig_XOR_2n2layer}(a) and (b) show respectively the learned decision surfaces
for \texttt{KARnet} and \texttt{feedforwardnet}. These results show the capability of
\texttt{KARnet} to fit the nonlinear surface and the premature stopping of
\texttt{feedforwardnet} for this data set.

Next, we compare the two networks using a five-layer architecture with each layer having
two hidden nodes for both networks. Fig.~\ref{fig_XOR_2n10layer}(a) and (b) show
respectively the decision surfaces for \texttt{KARnet} and \texttt{feedforwardnet}. These
results show the fitting capability of \texttt{KARnet} despite the larger number of
adjustable parameters and again, the premature stopping of \texttt{feedforwardnet} for
this data set.

\begin{figure}[hhh]
  \begin{center}
\begin{tabular}{c}
  \epsfxsize=9cm
  \epsffile[29     4   383   292]{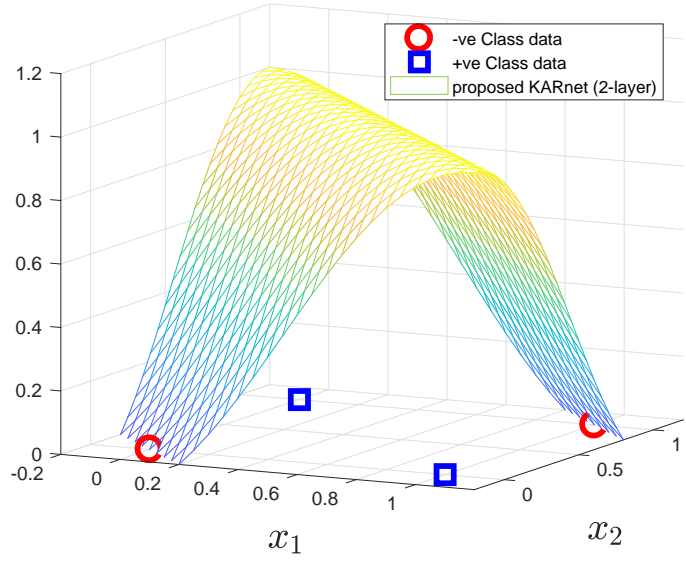}
  \\ \hspace{3mm} (a) \\*[5mm]
  \epsfxsize=9cm
  \hspace{0cm}
  \epsffile[29     4   381   293]{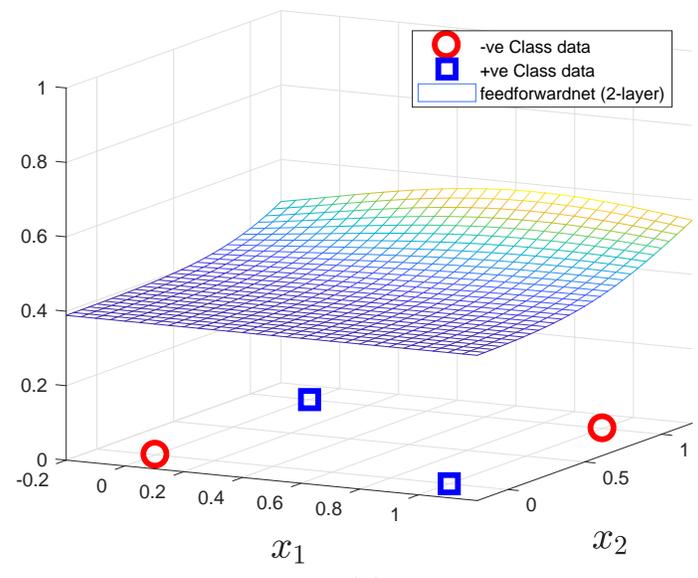}
  \\ \hspace{3mm} (b)
\end{tabular}
  \caption{Decision surfaces of two-layer feedforward networks
  (the proposed \texttt{KARnet} with $\f=softplus$, and \texttt{feedforwardnet} with $\f=softplus$ trained by \texttt{trainlm}).}
  \label{fig_XOR_2n2layer}
  \end{center}
\end{figure}

\begin{figure}[hhh]
  \begin{center}
\begin{tabular}{c}
  \epsfxsize=9cm
  \epsffile[28     4   381   292]{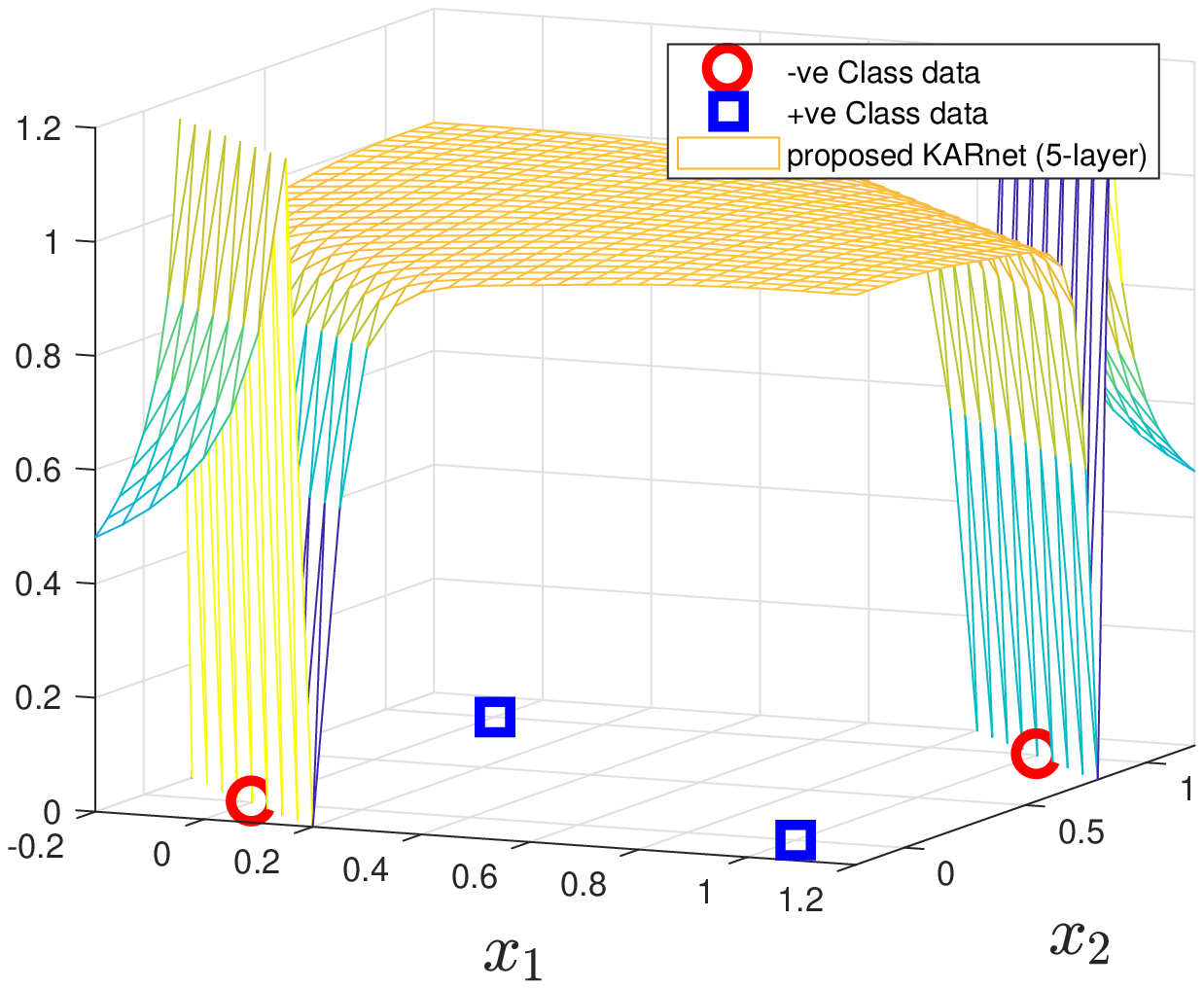}
  \\ \hspace{3mm} (a) \\*[5mm]
  \epsfxsize=9cm
  \hspace{0cm}
  \epsffile[28     4   381   292]{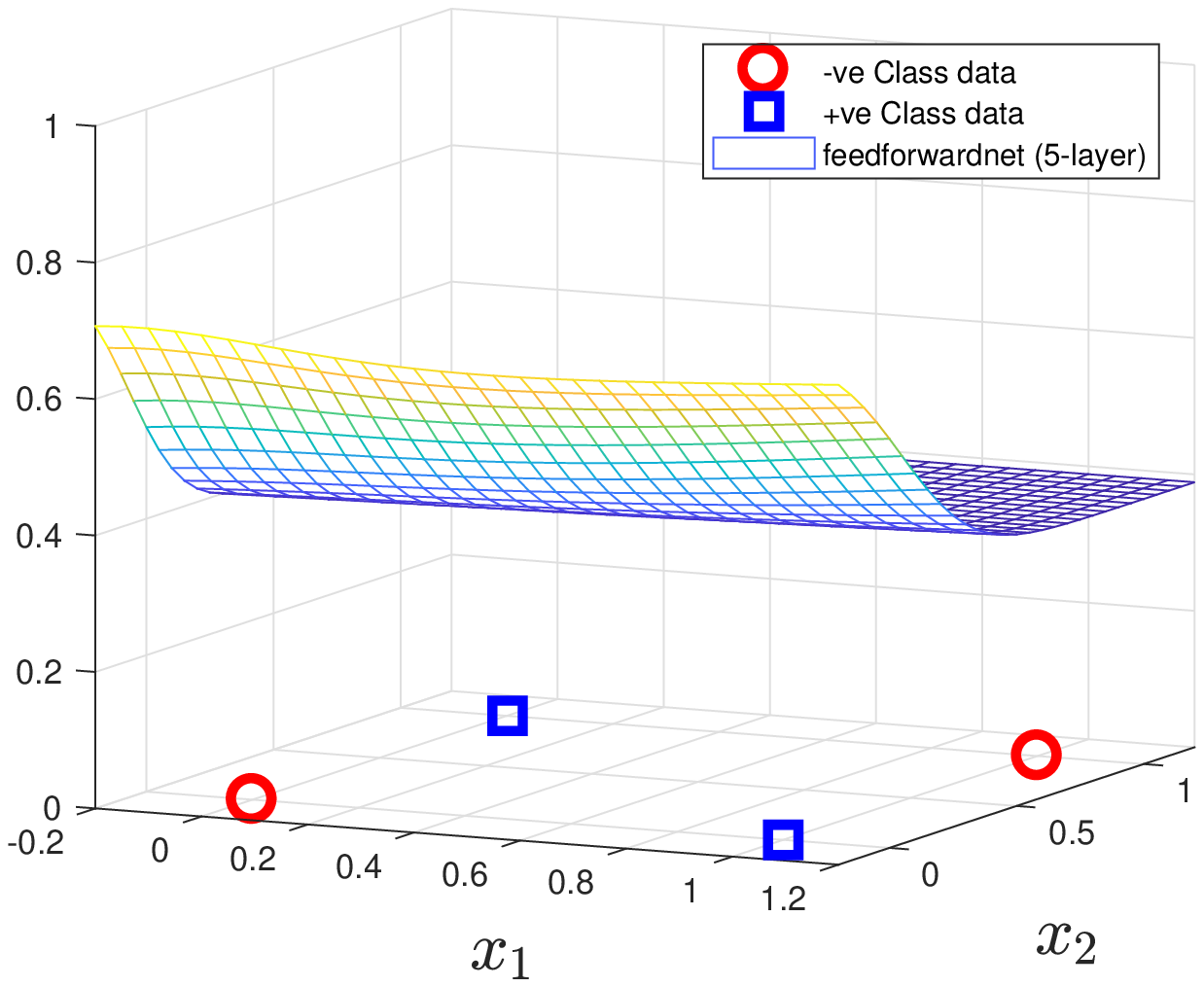}
  \\ \hspace{3mm} (b)
\end{tabular}
  \caption{Decision surfaces of five-layer feedforward networks
  (the proposed \texttt{KARnet} with $\f=softplus$, and \texttt{feedforwardnet} with $\f=softplus$ trained by \texttt{trainlm}).}
  \label{fig_XOR_2n10layer}
  \end{center}
\end{figure}

\subsection{The Three-Spiral Problem}

In this example, a total of 1500 randomly perturbed data points which form a 3-spiral
distribution have been used as the training data. Among these data, each of the spiral
arm consists of 500 data points (which are shown as red, green and blue circles in
Fig.~\ref{fig_Spiral_3layer}). A two-layer \texttt{KARnet} with 100 hidden nodes has been
adopted for learning these data points with respective labels using an indicator matrix.
Since there are three classes, the number of output nodes is 3. The learned decision
regions (which are shown in light red, green and blue tones) as shown in
Fig.~\ref{fig_Spiral_3layer} show the mapping capability of \texttt{KARnet} for the
three-category problem.

\begin{figure}[hhh]
  \begin{center}
  \epsfxsize=11.8cm
  \epsffile[36    19   385   295]{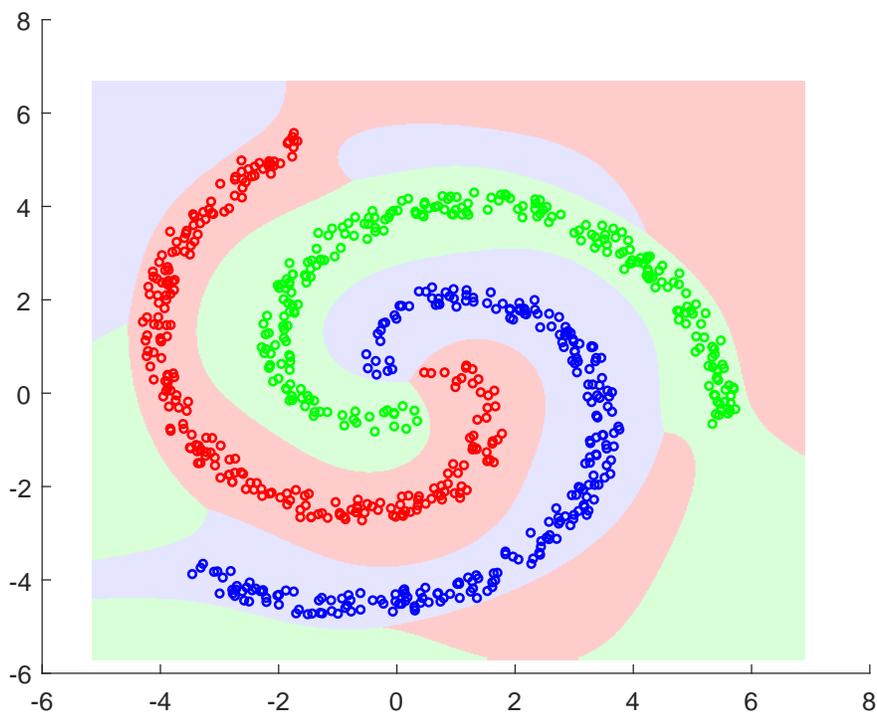}
  \caption{Decision regions of a three-layer feedforward network
  (proposed \texttt{KARnet} with $\f=softplus$) trained by the
  proposed method.}
  \label{fig_Spiral_3layer}
  \end{center}
\end{figure}

\clearpage

\section{Experiments on Real-World Data}

\subsection{Nursery Data Set}

The goal in this database \cite{Olave1,UCI1b} was to rank applications for nursery
schools based upon attributes such as occupation of parents and child's nursery, family
structure and financial standing, as well as the social and health picture of the family.
The eight input features for the 12960 instances are namely, `parents' with attributes
usual, pretentious, great\_pret; `has\_nurs' with attributes proper, less\_proper,
improper, critical, very\_crit; `form' with attributes complete, completed, incomplete,
foster; `children' with attributes  1, 2, 3, more; `housing' with attributes convenient,
less\_conv, critical; `finance' with attributes convenient, inconv; `social' with
attributes non-prob, slightly\_prob, problematic; `health' with attributes recommended,
priority, and not\_recom. These input attributes are converted into discrete numbers and
normalized to the range (0,1]. The output decisions include `not\_recom' with 4320
instances, `recommend' with 2 instances, `very\_recom' with 328 instances, `priority'
with 4266 instances and `spec\_prior' with 4044 instances. Since the category `recommend'
has not enough instances for partitioning in 10-fold cross-validation, it is merged into
the `very\_recom' category. We thus have 4 decision categories for classification.

For \texttt{KARnet}'s hidden parameter tuning, an inner 10-fold cross-validation loop
using only the training set was adopted to determine the hidden node size among
$h\in\{$1, 2, 3, 5, 10, 20, 30, 50, 80, 100, 200, 500$\}$. For 3-layer and 4-layer
networks, the network structures of $2h$-$h$-$q$ and $4h$-$2h$-$h$-$q$ are adopted
respectively ($q$ is the output dimension). The chosen hidden node size is then applied
for 10 runs of test evaluation using the outer cross-validation loop. The results for
2-layer, 3-layer, and 4-layer networks are respectively 92.39\% at $h=100$, 92.64\% at
$h=80$, and 92.73\% at $h=200$. These results are comparable with 98.89\% for the
\texttt{feedforwardnet} ($h=100$, 2-layer) and 91.69\% for the TERRM method \cite{Toh70}.

\subsection{Letter Recognition}

The data set comes with 20,000 samples, each with 16 feature attributes. The goal is to
recognize the 26 capital letters in the English alphabet based on a large number of
black-and-white rectangular pixel displays. The character images consist of 20 different
fonts where each letter within these 20 fonts was randomly distorted to produce a large
pool of unique stimuli \cite{frey1991letter,UCI1b}. Each stimulus was converted into 16
primitive numerical attributes such as the statistical moments and the edge counts. These
attributes were then scaled to fit into a range of integer values from 0 to 15.

Similar to the above evaluation setting, 10 trials of 10-fold stratified cross-validation
have been performed for classifying the 26 categories. The results for 2-layer, 3-layer,
and 4-layer \texttt{KARnet}s are respectively 88.99\%, 94.32\%, and 94.12\%, all at
$h=500$. The \texttt{feedforwardnet} (2-layer) encounters ``out of memory'' for the
current computing platform (Intel i7-6500U CPU at 2.59GHz with 8G of RAM).

\subsection{Optical Recognition of Handwritten Digits}

This data set was collected based on a total of 43 people, wherein 30 contributed to the
training set and different 13 to the test set \cite{kaynak1995methods,UCI1b}. The
original 32$\times$32 bitmaps were divided into non-overlapping blocks of 4$\times$4
where the number of on pixels were counted within each block. This generated an input
matrix of 8$\times$8 where each element was an integer in the range [0, 16]. The
dimensionality (64) is thus reduced (from 32$\times$32) and the resulted image is
invariant to minor distortions. The total number of samples collected for training and
testing are respectively 3823 and 1797. In our experiment, these two sets (training and
test sets) of data are combined for the running of 10 trials of 10-fold cross-validation
tests. Fig.~\ref{fig_Optdigit} shows some samples of the image data taken from the
training set and the testing set, respectively.

\begin{figure}[hhh]
  \begin{center}
  \epsfxsize=11.8cm
  \epsffile[58    39   377   292]{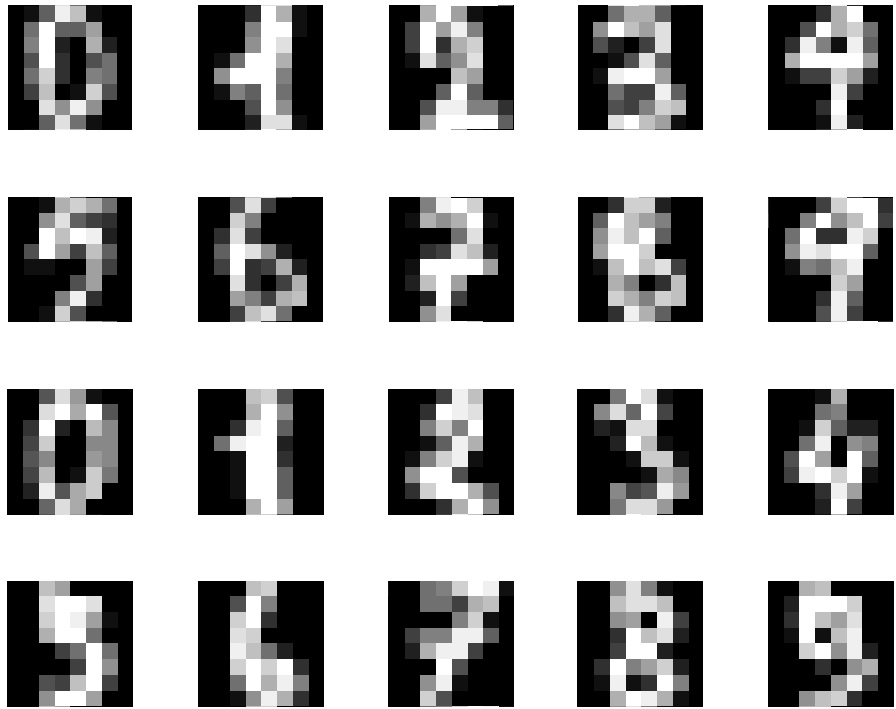}
  \caption{Handwritten digits: samples in the upper two panels are taken from the training set
  and samples in the bottom two panels are taken from the test set.}
  \label{fig_Optdigit}
  \end{center}
\end{figure}

The results for 2-layer, 3-layer, and 4-layer networks are respectively 97.25\% at
$h=500$, 97.17\% at $h=200$, and 96.96\% at $h=100$. These results are comparable with
the 96.81\% for the TERRM method \cite{Toh70}. The \texttt{feedforwardnet} (2-layer)
encounters ``out of memory'' for the current computing platform.

\subsection{Comparison with State-of-the-arts}

The results of \texttt{KARnet} are compared with several state-of-the-art methods namely,
the Reduced Multivariate Polynomial method (RM, \cite{Toh39}), the Total Error Rate
method adopting RM (TERRM, \cite{Toh70}), the Support Vector Machines adopting Polynomial
(SVM-Poly, \cite{LIBSVM}) and Radial Basis Function (SVM-Rbf, \cite{LIBSVM}) kernels, and
the \texttt{feedforwardnet} (2-layer) from the Matlab toolbox \cite{Matlab}, all running
under the similar 10 trials of 10-fold cross-validation protocol.
Table~\ref{table_comparison} shows that the proposed \texttt{KARnet} has comparable
prediction accuracy with state-of-the-art methods. While the SVMs had been tuned by
adjusting the kernel parameters (such as the order in the polynomial kernel, and the
Gaussian width in the radial basis kernel), the proposed network had been tuned by
adjusting the number of hidden nodes ($h$) in each layer according to the structures
$2h$-$h$-$q$ and $4h$-$2h$-$h$-$q$.

\begin{table}[hhh]{\normalsize
\begin{center}
\caption{Comparison of accuracy (\%) with state-of-the-arts} \label{table_comparison}
\begin{tabular}{|l|c|c|c|} \hline
  Methods                               &  Nursery  & Letter & Optdigit \\ \hline
  RM \cite{Toh39}                       & 90.93 & 74.14 & 95.32 \\
  TERRP \cite{Toh77}                    & 96.46 & 88.20 & 98.16 \\
  TERRM \cite{Toh70}                    & 91.69 & 78.42 & 96.81 \\
  SVM-Poly \cite{LIBSVM}                & 91.61 & 77.22 & 95.52 \\
  SVM-Rbf \cite{LIBSVM}                 & 98.24 & 97.14 & 99.13 \\
  \texttt{FFnet}(2-layer) \cite{Matlab} & 98.89 & OM    & OM \\
  \texttt{KARnet}(2-layer)              & 92.39 & 88.99 & 97.25 \\
  \texttt{KARnet}(3-layer)              & 92.64 & 94.32 & 97.17 \\
  \texttt{KARnet}(4-layer)              & 92.73 & 94.12 & 96.96 \\
\hline
\end{tabular}\\*[2mm]
\end{center} }
\begin{tabular}{lll}
\hspace{3mm} & \texttt{FFnet} &: \texttt{Feedforwardnet} from Matlab \cite{Matlab}.\\
\hspace{3mm} & OM &: Out of memory for the current computing platform.
\end{tabular}
\end{table}

%\newpage

\section{Conclusion}

In this article, the solution based on the manipulation of the kernel and the range space
has been found to be equivalent to that obtained by the least squares error estimation.
By exploiting this observation, a learning approach based on the kernel and the range
space manipulation has been introduced. The approach solves the system of linear
equations directly by exploiting the row and the column spaces without the need for error
formulation and gradient descent. The adoption of the learning approach to deep networks
learning validated its feasibility. The learning results of synthetic and real-world data
provided not only the numerical evidence but also the insights regarding the fitting
mechanism. This opens up the vast possibilities along the research direction.

\section*{Acknowledgment}

This research was supported by Basic Science Research Program through the National
Research Foundation of Korea (NRF) funded by the Ministry of Education, Science and
Technology (Grant number: NRF-2015R1D1A1A09061316).\\

\end{document}